\documentclass[sigconf]{acmart}

\usepackage{booktabs} 

\usepackage{algorithm}
\usepackage{algorithmic}
\usepackage{amsmath}
\usepackage{amssymb}
\usepackage{color}
\usepackage{paralist}
\usepackage{enumitem}
\usepackage{graphicx}
\usepackage{mathtools}
\usepackage{mdframed}
\usepackage{stackengine}
\usepackage{tikz}
\usepackage{xcolor}
\usepackage{xspace}

\usepackage{hyperref}

\hypersetup{
  colorlinks=true,
  allcolors=blue
}

\newcommand{\cA}{\mathcal{A}}
\newcommand{\cB}{\mathcal{B}}
\newcommand{\cI}{\mathcal{I}}

\newcommand{\cP}{\mathcal{P}}

\newcommand{\nat}{\mathbb{N}}

\newcommand{\rationals}{\mathbb{Q}}
\newcommand{\primes}[1]{{\mathcal{P}}_{#1}}

\newcommand{\integers}{\mathbb{Z}}
\newcommand{\rat}{\mathbb{Q}}

\newcommand{\ud}{\stackrel{\rm\scriptscriptstyle def}{=}}
\renewcommand{\ud}{\coloneqq}
\newcommand{\paths}{\text{paths}}

\newcommand{\initial}{\text{i}}
\newcommand{\final}{\text{f}}
\newcommand{\vect}{\text{v}}
\newcommand{\argmax}{\mathop{\mathrm{argmax}}}

\newcommand{\source}{\text{s}}
\newcommand{\dest}{\text{d}}
\newcommand{\chr}{\mathit{cr}}
\newcommand{\mutp}{\mathit{mp}}

\newcommand{\conf}[1]{{\color[rgb]{0.4,0.4,0.4} \,$\pm$ #1}}
\newcommand{\maxi}[1]{\:({\bf #1})}

\newcommand{\size}[2]{{\fontsize{#1}{9}\selectfont #2}}
\newcommand{\smallf}{\text{\size{5.5}{$f$}}}
\newcommand{\fmapsto}{\stackon[-0.1pt]{\,$\mapsto$\,}{\,\smallf}}

\DeclarePairedDelimiter\floor{\lfloor}{\rfloor}

\setcopyright{rightsretained}

\begin{document}
\title{Genetic Algorithm for the Weight Maximization Problem on Weighted Automata}

\author{Elena Guti\'{e}rrez}
\orcid{0000-0001-5999-7608}
\affiliation{%
  \institution{IMDEA Software Institute, Spain}
}
\affiliation{%
  \institution{Universidad Polit\'{e}cnica de Madrid, Spain}
}
\email{elena.gutierrez@imdea.org}

\author{Takamasa Okudono}
\affiliation{%
  \institution{National Institute of Informatics, Japan}
}
\affiliation{%
  \institution{The Graduate University for Advanced Studies, Japan}
}
\email{tokudono@nii.ac.jp}

\author{Masaki Waga}
\affiliation{%
  \institution{National Institute of Informatics, Japan}
}
\affiliation{%
  \institution{The Graduate University for Advanced Studies, Japan}
}
\email{mwaga@nii.ac.jp}

\author{Ichiro Hasuo} 
\affiliation{%
  \institution{National Institute of Informatics, Japan}
}
\affiliation{%
  \institution{The Graduate University for Advanced Studies, Japan}
}
\email{hasuo@nii.ac.jp}

\renewcommand{\shortauthors}{E. Guti\'{e}rrez et al.}

\begin{abstract}
  The weight maximization problem (WMP) is the problem of finding the word of highest weight
  on a weighted finite state automaton (WFA).
  It is an essential question that emerges in many optimization
  problems in automata theory.
  Unfortunately, the general problem can be shown to be undecidable, whereas its bounded
  decisional version is NP-complete.
  Designing efficient algorithms that produce approximate solutions to the WMP in reasonable time
  is an appealing research direction that can lead to several new applications including
  formal verification of systems abstracted as WFAs.
  In particular, in combination with a recent procedure that translates a recurrent neural network into a
  weighted automaton, an algorithm for the WMP can be used to analyze and verify the network
  by exploiting the simpler and more compact automata model.

  In this work, we propose, implement and evaluate a metaheuristic based on genetic algorithms to
  approximate solutions to the WMP.
  We experimentally evaluate its performance on examples from the literature and show
  its potential on different applications.
\end{abstract}
%
%
\ccsdesc[500]{Theory of computation~Quantitative automata}
\ccsdesc[500]{Theory of computation~Theory of randomized search heuristics}

\keywords{Weighted automata, Genetic algorithms, Metaheuristics, Recurrent neural networks}

\copyrightyear{2020}
\acmYear{2020}
\acmConference[GECCO '20]{Genetic and Evolutionary Computation Conference}{July 8--12, 2020}{Canc\'{u}n, Mexico}
\acmBooktitle{Genetic and Evolutionary Computation Conference (GECCO '20), July 8--12, 2020, Canc\'{u}n, Mexico}
\acmDOI{10.1145/3377930.3390227}
\acmISBN{978-1-4503-7128-5/20/07}

\maketitle

\begin{acks}
Part of the research was conducted during E.G.'s internship at National Institute of Informatics, Japan.  
The authors are supported by ERATO
HASUO Metamathematics for Systems Design Project (No.~\grantnum{}{JPMJER1603}), JST, and by JSPS Grants-in-Aid No.~\grantnum{}{15KT0012} \&~\grantnum{}{18J22498}.
E.G. is also supported by \grantnum{}{BES-2016-077136} grant from the Spanish Ministry of Economy, Industry and Competitiveness.
\end{acks}

\section{Introduction}

\paragraph{Background}
Finite-state automata are transition systems that \emph{accept} or \emph{reject} words from a given alphabet of symbols.
A more general notion is that of \emph{weighted} finite-state automata (WFA).
These are automata where transitions and states are augmented with a weight from the real numbers.
Thus, they do not simply accept or reject words, but they induce a real function from words over the alphabet to weights.
The weight of a word is computed by adding together the weights of all executions that are labeled with that word, where the weight of each single execution
is obtained by multiplying the weights of the transitions that composes it, together with the weight of the initial and final state of the execution.

The WFA model has been extensively studied in the literature~\cite{Droste2009} and has found a great number of modern applications in speech recognition~\cite{Mohri97}, digital image compression~\cite{BaderHS04,Hafner98}, sequence prediction~\cite{CortesHM04} and optical character recognition~\cite{KolakBR03}; as well as in formal verification where they are used for the verification of quantitative systems~\cite{ChatterjeeDH08,DrosteG05,Schutzenberger61b}.
These applications have also enhanced the use of \emph{weighted automata learning} to abstract more complex systems as WFAs over the reals by approximating a real-valued target function, using as training sample a finite set of pairs of words and target values~\cite{Balle2015}.
Recently, this technique has been successfully used to extract WFAs from real-output recurrent neural networks (RNN)~\cite{tmp-Takamasa2019}.
The result is a simpler, compact and more interpretable transition system (compared to the original RNN)
in which inference can be performed up to three orders of magnitude faster, as empirical results show~\cite{tmp-Takamasa2019}.
Furthermore, these systems admit particular optimization techniques for their analysis such as \emph{memoization} of partial executions.

\paragraph{The problem}
In this work, we are interested in the problem of finding the word with the highest weight in a WFA, known as the \emph{weight maximization problem} (WMP).
This problem has drawn attention before, especially in the context of natural language processing and speech recognition~\cite{Mohri02,Higuera13,Higuera14}.
In that case, the subclass of \emph{probabilistic} WFA is typically used to represent different levels of the recognition task.
Thus, the highest-probability string corresponds to the most likely translation for an observed input sequence.
Several algorithms have been proposed to solve this problem, but unfortunately they usually rely on the special structure of probabilistic WFAs, which enables techniques that are out of hand if one considers more general weighted automata.
In fact, we are interested in WFAs such as the ones that are obtained as output of real-weighted automata learning techniques.
In particular, we will show the potential application of an algorithm to solve the WMP when combined with the  WFA extraction procedure for RNNs\cite{tmp-Takamasa2019} to enable a more efficient ``light-weight verification" of the network.

When tackling the WMP, it is easy to observe that if the range of the weight function described by the WFA is unbounded  over the reals, then the notion of \emph{the word with the highest weight} is somehow vague.
Furthermore, it is well-known that even if one bounds the range of the weight function, the associated decision problem is undecidable~\cite{Paz1971,Blondel2003,Gimbert2010}.
In consequence, we bound the domain of the problem, namely, we restrict our search to words of \emph{bounded length}.
This decision is also supported by the idea that for verification purposes, we are not interested in extremely long witnesses (words), in general.
We call this bounded version of the problem, the \emph{bounded weight maximization} problem (BWMP) over WFAs.

\paragraph{Our Contribution}
First, we present a metaheuristic based on the genetic algorithm (GA) for approximating a solution to the bounded weighted maximization problem.
Second, we experimentally evaluate the algorithm on examples from the literature and show its potential for estimating the error between a WFA extracted from an automata learning technique and the target function.
In particular, we will focus our attention on WFAs that result from the extraction procedure applied to RNNs~\cite{tmp-Takamasa2019}.
As an independent contribution we show that the bounded WMP is NP-complete by means of a reduction from the Hamiltonian path problem,
even in the case where weights are over $\integers$ and the alphabet contains only two symbols.
Next we give further details on each of the contributions.

The choice of a genetic-algorithm-based metaheuristic for this purpose is motivated by the idea that it is simple and natural to code elements of the language of a WFA as strings over an alphabet of symbols, with its weight as the fitness function.
Also, the way in which words are \emph{read} in a WFA by repeating patterns, i.e., paths and cycles in its finite structure, suggests that higher-weighted words might be those repeating higher-weighted patterns in the WFA.
This intuition resembles to the fundamental basis of genetic algorithms that relies on the idea that combining building blocks allows to build strings with higher expected performance.
In order to improve the time performance of the algorithm, we exploit the matrix representation of the WFA model.
Using this representation, a WFA is described as a finite set of transition matrices, one for each alphabet symbol, with values over the reals indicating the weight of each transition, plus an initial and final real vector indicating the weight of the initial and final states.
Thus, computing the weight of a word amounts to multiply the set of matrices corresponding to the sequence of symbols given by the word together with the initial and final vector.
We show the potential of using \emph{memoization} to accelerate the computation of these products relying on the associativity of matrix multiplication, a technique that is not at hand if one uses directly the RNN weight function (instead of the WFA) as fitness function of the GA.
Notice that a genetic-algorithm-based metaheuristic provides several good approximations of the word with the highest weight (in general, as many as the size of the population).
This makes the approach convenient for our  purpose of applying our algorithm to the verification of RNN as we may obtain a number of witnesses of the behavioral difference between the RNN and the specification.

We perform an empirical evaluation of our algorithm conducted by the answer of two main questions:
\begin{inparaenum}[\upshape(1)]
\item is our genetic-algorithm-based metaheuristic a good fit for approximating the BWMP?
\item given that memoization of partial executions is an appropriate optimization for WFAs, what is an lower bound on the gain of using this technique?
\end{inparaenum}

Regarding the first question, we compare the performance of our algorithm against random search.
The results shows that our method outperforms the latter, concluding that, since WFAs are not completely black-boxes, the GA is able to successfully exploit to some extent the internal structure of these devices.
Furthermore, we evaluate the quality of the solutions found by our algorithm by designing an experiment where the word with the highest weight is known.
We observe that in many cases our algorithm finds the word with the highest weight, and we identify the hardest cases as those where the WFA has several local maxima corresponding to words that significantly differ from the word with the highest weight.

On the second question, our experiments show that by using a simple memoization technique the number of words analyzed per second is more than 3 times that of a non-optimized algorithm, which evidences the potential of this technique in the context of WFAs.

Finally, we conduct a case study in order to solve a third research question:
\begin{inparaenum}[\upshape(3)]
\item how to exploit our algorithm for performing light-weight verification of RNNs?
\end{inparaenum}

We first show that an algorithm to approximate the WMP allows to compute a simple notion of \emph{distance} between WFAs.
Namely, given two WFAs \(\cA\) and \(\cB\), we define the distance between \(\cA\) and \(\cB\) as the weight that maximizes the difference between the two automata in absolute value.
Intuitively, this notion of distance gives the maximum ``error'' of \(\cA\) approximating \(\cB\).
In combination with the WFA extraction procedure for RNNs, this enables to perform a  light-weight verification of the RNN as it allows to estimate the error between the RNN behavior and the function it models, as well as to identify input sequences whose output significantly differs the specification, i.e., input sequences that are misclassified by the RNN.
It is worth to notice that the accuracy of the WFA extracted from the RNN is high~\cite{tmp-Takamasa2019}, but yet an \emph{approximation} and thus, despite of being practical to gain confidence about the correctness of the network, it is especially useful to detect misclassified input sequences.
Also, note that WFAs are strictly less expressive than RNNs as they recognize the subclass of weighted \emph{regular} languages.
Therefore, the verification of the RNN we propose is always against a weighted regular specification, and not \emph{any} possible specification.
As a case of study satisfying this condition, we show that an RNN trained to learn a weighted regular language, namely a weighted variant of the language of well-parenthesized words, satisfies its specification on a bounded-length set of words.
Another application of our algorithm in the context of network analysis is that of determining whether, given a prefix of an RNN input, there exists a sequence of input symbols that can be read after and makes the RNN output reach a certain threshold value.
With an RNN abstracted as a WFA, one can compute the set of states that are reachable reading the given prefix and define them as the new initial states of the automaton.
Then, apply the algorithm for solving the WMP in the latter to approximate the maximum weight and compare it with the threshold.
This has particular application in RNNs that are trained for anomaly detection~\cite{LvWYL18,XiaoZMHS19}.

\paragraph{Related Work}
The problem of determining the \emph{best} string on a weighted automata, i.e., the word with the highest weight, has been previously studied in the literature for the subclass of \emph{generative}\footnote{This is in contrast to the definition of probabilistic automata first introduced by Rabin~\cite{Rabin63} and the one we will use in this paper, which interprets automata as \emph{accepting} devices. For generative probabilistic WFAs, the probability of a word \(w\) can be interpreted as the probability of reaching a final state when \(w\) is used as a \emph{scheduling policy}.} \emph{probabilistic} WFAs~\cite{Mohri02,Higuera13,Higuera14}.
These are WFAs that define a probability distribution over the words of the alphabet and they typically emerge in the context of speech recognition and natural language applications.
Mohri and Riley~\cite{Mohri02} present an efficient algorithm for solving the \(n\)-best-string problem, in order to determine the best hypothesis (or \(n\)-best different hypotheses) among all those considered by the recognizer.
Their method relies on two general algorithms for WFAs: determinization and a general \(n\)-shortest-path algorithm.
However, their WFAs are always acyclic, so-called \emph{lattices}, typically used in speech recognition (as opposed to our more general class of WFAs which typically contain cycles, especially those coming from the extraction procedure for RNNs), which guarantees the termination of the determinization procedure.
Note that if our WFAs were deterministic, then the WMP would become tractable since it amounts to compute the \emph{longest} path in a weighted graph from an initial to a final state.
However, the weighted automata we are interested in are not deterministic in general, not even \emph{determinizable}.
Characterizing classes of weighted automata with real weights that admit an equivalent deterministic version is an interesting
research direction on which not much progress has been made~\cite{Mohri2003}.

De la Higuera and Oncina~\cite{Higuera13,Higuera14} also tackle the problem of finding the best string in generative probabilistic weighted automaton, motivated by natural language processing applications.
On the other hand, Casacuberta and De la Higuera~\cite{CasHig2000} proved that the threshold reachability problem for probabilistic weighted automata
is NP-hard even when you bound the length of the most probable string.
They show how to reduce any instance of the \emph{satisfiability problem} to the best string problem on a generative probabilistic automata
with four alphabet symbols.
We pursue the study of the hardness of the best string problem and show that the problem is NP-complete in the case of
automata with weights over $\integers$ and just two alphabet symbols (see Section~\ref{app:np-completeness} in the Appendix).
Given the hardness of these problems, it is remarkable that De la Higuera and Oncina~\cite{Higuera13}
provide an algorithm for the best string problem whose efficiency depends on the probability of the best (or most probable) string itself.
However, this result strongly relies on the probabilistic nature of this subclass of WFAs.
In the case of general WFAs over the rational numbers, to the best of our knowledge, no other algorithm for the input maximization has been proposed before.

Some works explore other notions of distance between WFAs, e.g., the \emph{bisimulation metric}~\cite{Balle17},
a non-computable distance that is based on the joint spectral radius of the transition matrices of the WFA.
However, this notion does not fit our purposes as it describes a universal property rather than an existential property.
Namely, it defines a property of two automata w.r.t. all possible words over an alphabet rather than identifying a single word
with a certain property.

\section{Preliminaries}
\label{section:preliminaries}
\subsection{Languages}
Let \(\Sigma\) be a finite \emph{alphabet} of symbols.
A \emph{word} \(w\) over the alphabet \(\Sigma\) is a finite sequence of symbols \(w \ud a_1 \cdots a_{n}\) with \(a_i \in \Sigma\) for each \(i \in \{1, \ldots, n\}\).
In that case, we say that \(n\) is the \emph{length} of \(w\) and we denote it by \(|w| = n\).
We define the language \(\Sigma^*\) as the set of all words over the alphabet \(\Sigma\), including the \emph{empty string} \(\varepsilon\), whose length is \(0\).
Given a natural value \(k \geq 1\), we define the finite language \(\Sigma^{\leq k} \ud \{w \in \Sigma^* \mid 1 \leq |w| \leq k\}\).
We define the \emph{size} of a language \(L \subseteq \Sigma^*\), denoted by \(|L|\), as the cardinality of the set \(L\).
Finally, given two words \(v = a_1 \cdots a_{n_1}, w = b_1 \cdots b_{n_2} \in \Sigma^*\), \(v\cdot w \ud a_1 \cdots a_{n_1}\cdot b_1 \cdots b_{n_2}\) denotes the word that results from concatenating \(w\) after \(v\).

\subsection{Matrices and Vectors}
Despite of the fact that our model of WFAs is defined over the real semiring, for computational reasons we will give definitions over the \emph{rational} numbers.
Given a column vector \(\vect \in \rationals^{d}\), with \(d \geq 1\), we denote by \(\vect^\top\) the \emph{transpose} of \(\vect\).
Given \(d_1, d_2 \geq 1\), we denote by \(0_{d_1 \times d_2}\) the matrix of dimension \(d_1 \times d_2\) with all its entries equal to \(0\).
If \(d_1 = d_2\), we simply write \(0_{d_1}\).
Similarly, \(\overline{0}_d\) denotes the column vector of dimension \(d \geq 1\) with all its entries equal to \(0\).
\subsection{Weighted Automata}
A \textit{weighted automaton} \(\cA\) (WFA for short) over \(\rationals\) is a 5-tuple \(\cA = (Q, \Sigma, \{M_{a}\}_{a \in \Sigma}, \initial, \final)\) where \(Q\) is a finite set of \emph{states}, \(\Sigma\) is a finite alphabet of symbols, \(\initial\) and \(\final\) are both column vectors in \(\rationals^{|Q|}\) called the \emph{initial} and the \emph{final vector} respectively, and, for each \(a \in \Sigma\), \(M_a\) is a square matrix in \(\rationals^{|Q| \times |Q|}\) called the \emph{transition matrix of \(a\)}.

The \emph{weight} of a word \(w \in \Sigma^*\) w.r.t. \(\cA\) is defined as follows.
Let \(w = a_1\,a_2\cdots a_n\) with \(a_i \in \Sigma\):
\begin{equation}
\label{eq:weight_word}
W_{\cA}(w) \ud \initial^{\top} \cdot \prod_{i=1}^{n} M_{a_i} \cdot \final \enspace ,
\end{equation}
Now we give an alternative way to define the weight of a word that, instead of using the matrix representation of the WFA, uses its description as a transition system, with the only purpose of providing further intuition on this notion.
First, let us give some previous definitions and notation.
For each \(q, q' \in Q\), we will denote the \(q\)-component of a vector \(\vect \in \rationals^{|Q|}\) by \(\vect(q)\), and the \((q,q')\)-entry of a matrix \(M \in \rationals^{|Q| \times |Q|}\) by \(M(q, q')\).
Define the \emph{transition set} of \(\cA\) as \(\delta \ud \{(q,a,q') \mid M_a(q,q') \neq 0 \} \subseteq Q \times \Sigma \times Q\) and the \emph{weight of a transition} \(\sigma = (q,a,q')\) as \(W_\cA(\sigma) \ud M_a(q,q')\).
We say that \(q\) and \(q'\) are the \emph{source} and the \emph{destination} state of \(\sigma\), respectively.
The transition set of \(\cA\) allows us to define the notion of a \emph{path} of \(\cA\) as follows.
A \emph{path} \(\pi = (q_0, a_1, q_1)(q_1, a_2, q_2) \cdots (q_{n-1}, a_{n}, q_n)\) \((n\geq 1)\) of a \(\cA\) is an element of \(\delta^*\) of \emph{consecutive} transitions, i.e., the destination state of every transition in the sequence (except the last one) coincides with the source state of the next transition in the path.
In the latter case, we say that \(\pi\) \emph{reads} the word \(a_1\,a_2\cdots a_{n}\).
Given \(S,D \subseteq Q\) and \(w \in \Sigma^*\), denote by \(\paths_\cA(S,w,D)\) the set of all paths from a state in \(S\) to a state in \(D\) reading \(w\).
Finally, when \(S = \{q \in Q \mid \initial(q) \neq 0\}\) and \(D = \{q \in Q \mid \final(q) \neq 0\}\), we denote \(\paths_\cA(S, w, D)\) simply by \(\paths_\cA(w)\).
Define the \emph{weight of a path} \(\pi = \sigma_1 \cdots \sigma_n\) as:
\[W_{\cA}(\pi) \ud \prod_{i=1}^{n} W_{\cA}(\sigma_i) \enspace .\]
Finally, define the weight of a word as follows:
\[W_{\cA}(w) \ud \sum_{\pi \in \paths_\cA(w)} W_{\cA}(\pi) \enspace .\]

We will often use the notation \(\cA(w)\) to denote the weight \(W_{\cA}(w)\).
We will refer to a subclass of WFAs, namely, \emph{probabilistic weighted automata} (PWFAs).
A PWFA \(\cP = (Q, \Sigma, \{M_{a}\}_{a \in \Sigma}, \initial, \final)\) over \(\rationals\) is a weighted automaton satisfying the following properties:
\begin{inparaenum}[\upshape(\itshape i\upshape)]
\item there is exactly one state \(q \in Q\) such that \(\initial(q) = 1\) and, for all \(q' \in Q\setminus\{q\}: \initial(q) = 0\),
\item for each \(q \in Q: \final(q) \in \{0,1\}\), and
\item for each \(q \in Q\) and  \(a \in \Sigma: \sum_{q'\in Q} M_a(q,q') = 1\).
\end{inparaenum}
If \(\final(q) = 1\) then \(q\) is an \emph{accepting} state, otherwise, it is \emph{non-accepting}.
\subsection{Operations on WFAs}
\label{app:operations}

Let us recall the following operations over WFAs that we will use in Section~\ref{section:applications}, and in the Appendix.
Let us define the following WFAs: \(\cA = (Q, \Sigma, \{M_a\}_{a \in \Sigma}, \initial, \final)\), \(\cA_1 = (Q_{1}, \Sigma, \{M^{(1)}_a\}_{a \in \Sigma}, \initial_1, \final_1)\) and  \(\cA_2 = (Q_2, \Sigma, \{M^{(2)}_a\}_{a \in \Sigma}, \initial_2, \final_2)\).

\begin{definition}[Unary subtraction WFA]
Given a WFA \(\cA\), define the WFA \(\ominus \cA \ud (Q, \Sigma, \{M_a\}_{a \in \Sigma}, -\initial, \final)\).
\end{definition}
Note that \(W_{\ominus \cA}(w) = -W_{\cA}(w)\), for each \(w \in \Sigma^*\).

\begin{definition}[Sum of WFAs]
Given WFAs \(\cA_1\) and \(\cA_2\), define the WFA \(\cA_1 \oplus \cA_2 \ud (Q_1 \cup Q_2, \Sigma, \{M'_a\}_{a\in\Sigma}, \initial', \final'))\) where \(\initial' \ud (\initial_1\,\,\initial_2)\), \(\final' \ud (\final_1\,\,\final_2)\) and
\begin{displaymath}
M'_a  \ud
\begin{bmatrix}
M^{(1)}_a  & 0_{d_1\times d_2} \\[0.6pt]
0_{d_2 \times d_1}  & M_{a}^{(2)}
\end{bmatrix},
 \text{ for each }a \in \Sigma \enspace .
\end{displaymath}
\end{definition}
Note that \(W_{\cA_1 \oplus \cA_2}(w) = W_{\cA_1}(w) + W_{\cA_2}(w) \), for each \(w \in \Sigma^*\).

\begin{definition}[Subtraction of WFAs]
Given WFAs \(\cA_1\) and \(\cA_2\), define the WFA \(\cA_1 \ominus \cA_2 \ud \cA_1 \oplus (\ominus \cA_2)\).
\end{definition}
Note that both WFAs \(\cA_1 \oplus \cA_2 \) and \(\cA_1 \ominus \cA_2 \) have \(|Q_{\cA}| + |Q_{\cB}|\) states.
\begin{definition}[Product of WFAs]
Given WFAs \(\cA_1\) and \(\cA_2\), define the WFA \(A \otimes B \ud (Q_1 \times Q_2, \Sigma, \{M'_a\}_{a\in \Sigma}, \initial', \final')\) where \(\initial' \ud \initial_1\otimes\initial_2, \quad \final' \ud \final_1\otimes\final_2 \quad \text{ and } \quad
M'_a  \ud M^{(1)}_a \otimes M^{(2)}_a \text{ for each }a \in \Sigma\),
where \(\otimes\) denotes the Kronecker product of matrices.
\end{definition}
Note that \(W_{\cA_1 \otimes \cA_2}(w) = W_{\cA_1}(w) \cdot W_{\cA_2}(w) \), for each \(w \in \Sigma^*\).

\begin{definition}[Sum of a WFA and a real]
Given a WFA \(\cA\) and \(\alpha \in \rationals\), define the WFA \(A \oplus \alpha \ud (Q \cup \{q_\alpha\}, \Sigma, \{M_a'\}_{a\in\Sigma}, \initial', \final')\) where \(Q \cap \{q_{\alpha}\} = \emptyset\), \(\initial' \ud (\initial \,\alpha)\), \(\final' \ud (\final\, 1)\) and
\begin{displaymath}
M_a'  \ud
\begin{bmatrix}
M_a  & \overline{0}^{\top}_{d} \\[0.6pt]
\overline{0}^{\top}_{d}  & \alpha
\end{bmatrix},
 \text{ for each }a \in \Sigma \enspace .
\end{displaymath}
\end{definition}
Notice that \(W_{\cA \oplus \alpha}(w) = W_{\cA}(w) + \alpha\), for each \(w \in \Sigma^*\).
Finally, note that all the binary operations defined below are commutative except from the product of WFAs.

\section{The Weight Maximization Problem}
\label{section:WMP}
We are interested in computing the word with the highest weight in a WFA with weights over \(\rationals\).
This is the so-called \textit{Weight Maximization Problem}.
\begin{definition}[Weight Maximization Problem]
Given a WFA \(\cA\), the \emph{Weight Maximization Problem} (WMP) consists in computing a word \(w_0\) such that \(W_{\cA}(w_0) \geq W_{\cA}(w)\), for all \(w \in \Sigma^*\).
\end{definition}

Observe that if the range of the weight function of a WFA is unbounded over \(\rationals\), then \(w_0\) might not exist.
Furthermore, even if the range of the function \(W_{\cA}\) is bounded, for instance, to the interval \([0,1]\) as in probabilistic WFAs, the associated decision problem\footnote{The decision problem associated to the WMP, namely, the Threshold Reachability Problem asks, for a given WFA \(\cA\) and a threshold \(\nu \in \rationals\), whether there exists a word \(w \in \Sigma^*\) s.t. \(W_{\cA}(w) \geq \nu\).} has been proved to be undecidable~\cite{Paz1971,Blondel2003,Gimbert2010}.
Therefore, the WMP turns out to be non-computable in general.

For this reason, we look at the problem that results from bounding the domain of the weight function of WFAs, namely, assuming that the length of the words is bounded by a fixed value.
Consequently, we define the so-called \emph{Bounded Weight Maximization Problem} as follows.

\begin{definition}[Bounded Weight Maximization Problem]
  Given a WFA \(\cA\) and \(k \geq 1\), the \emph{Bounded Weight Maximization Problem} (BWMP) consists of computing a word \(w_0 \in \argmax_{w\in \Sigma^{\leq k}}W_{\cA}(w)\).
\end{definition}

Note that the maximum of the weight function \(W_{\cA}\) always exists over the bounded set of words \(\Sigma^{\leq k}\).
The decision problem associated to the BWMP is defined as follows.

\begin{definition}[Bounded Threshold Reachability Problem]
\label{def:BTRP}
Given a WFA \(\cA\), \(k \geq 1\) and \(\nu \in \rationals\), the \emph{Bounded Threshold Reachability Problem} (BTRP) asks whether there exists a word \(w \in \Sigma^{\leq k}\) s.t. \(W_{\cA}(w) \geq \nu\).

\end{definition}
Clearly, this problem is decidable since the search space is finite.
Moreover, we prove that it is NP-complete for WFAs with weights over $\integers$ by means of a reduction from the Hamiltonian Path Problem (see Section~\ref{app:np-completeness} in the Appendix).
\begin{theorem}
  The BTRP with weights over $\integers$ is NP-complete.
\end{theorem}

Thus, we propose a metaheuristic based on the genetic algorithm to approximate a solution for the BWMP.

\section{Genetic Algorithm for the BWMP}
\label{section:algorithm}

The metaheuristic we propose follows the steps of the genetic algorithm~\cite{Goldberg89}, i.e., defines a set of genetic operators and routines that are executed following the baseline of this evolutionary technique in the way Figure~\ref{fig:flow-diagram} illustrates.

First, we show the genetic encoding of our problem.
\subsection{Genetic representation of the problem}
Given a WFA \(\cA = (Q, \Sigma, \{M_a\}_{a \in \Sigma}, \initial, \final)\) and \(k \geq 1\), we give a genetic representation of the solution domain for the BWMP as follows.
At each iteration or \emph{generation} of the algorithm, the set of candidate solutions, called the \emph{population}, evolves in order to find better candidates.
Each population is a set of \emph{individuals} which in our setting are words over the alphabet \(\Sigma\).
The sequence of symbols of each individual, usually known as \emph{chromosomes}, is of variable length which ranges in the interval \([1, k]\).
The \emph{fitness function} that evaluates on each individual \(w\) is defined as the weight function \(W_{\cA}\) applied on \(w\).
Thus, given two individuals \(u,v \in \Sigma^*\), we say that \(u\) is \emph{better than} \(v\) if{}f \(W_{\cA}(u) > W_{\cA}(v)\).

Next we give further details on the genetic operators and routines we use in our algorithm in Sections~\ref{subsection:initialization}~--~\ref{subsection:evaluation}.
All the parameter values defined in these routines will be fixed in the Section~\ref{section:experimental} of Experimental Results.

\begin{figure*}[tp]
  \begin{mdframed}
    \centering
    \includegraphics[width=\textwidth]{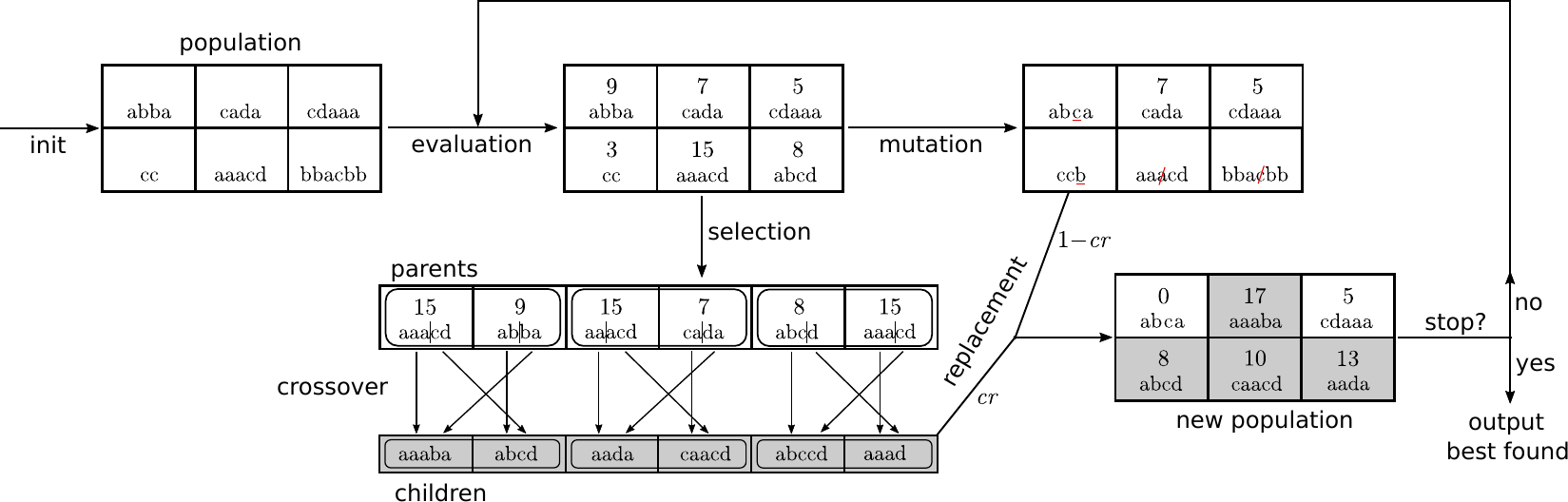}
    \caption{Flow diagram of our genetic algorithm for the WMP (toy example with $N = 6$ and $\chr = 2/3$).}
    \label{fig:flow-diagram}
  \end{mdframed}
\end{figure*}

\subsection{Initialization}
\label{subsection:initialization}
In this step, we create an initial population of fixed size, which we denote by \(N\), and keep constant throughout the execution.
We use a random initialization procedure that generates an initial population of individuals of length at most \(k\) uniformly at random.
We compared this approach to an alternative method that we further explain in Section~\ref{section:observations}.

\subsection{Selection}
\label{subsection:selection}
Selection is used at the \emph{crossover} and the \emph{replacement} step.
This method chooses \(N_s\) individuals from a population of size \(N\), following a \emph{fitness rank selection}.
Namely, we would like to have a sampling method that assigns more probability to words with higher weight.
For this, we represent the population as an ordered list from lower (smaller indices in the list) to higher (larger indices) fitness function.
We then define a distribution such that the probability assigned to the highest element of the list is \(\beta (> 1)\) times greater than
the probability of selecting the lowest.
In particular, if we also want a uniform gradient of probabilities, this sampling can be performed very efficiently
by sampling a uniform value $u \leftarrow [0,1]$, computing $i \coloneqq N\log_{\beta}{(1+u(\beta{-}1))}$ and
selecting the $i$-th individual of the sorted list.
This sampling method is repeatedly and independently performed until $N_s$ elements have been selected.

\subsection{Crossover}
\label{subsection:crossover}

In this step, we select \(N\) individuals according to the latter selection procedure.
Each of these individuals, called \emph{parents}, are combined by pairs to obtain two, possibly better, solutions.
This method works in a similar way as the  \emph{single-point crossover} operator in genetic algorithms.
Observe that, since the pair of parents might be of different length, the crossover point must be within the range of indices of both words.
Doing so, the pair of children produced will have the same lengths as their parents.
In order to allow for more diversity, we propose a slight variant of this method that, instead of one single index, chooses a pair of them, one for each individual.
Moreover, if any of the computed children has length greater than the bound \(k\), that individual is truncated ignoring the last symbols of the string.
This enforces the length constraint to hold throughout the execution.
Formally, given two parents \(v = v_1 \cdots v_{|v|},  w = w_1 \cdots w_{|w|}  \in \Sigma^{\leq k}\) selected to be recombined, our method chooses \emph{two} \emph{crossover points} \(i \in \{1, \ldots, |v|\}\) and \(j \in \{1, \ldots, |w|\}\) uniformly at random, and produces two \emph{children} \(x,y \in \Sigma^{\leq k}\) defined as \(x \ud v_1 \cdots v_i\cdot w_{j+1} \cdots w_{\ell_x}\) and \(y \ud w_1 \cdots w_j \cdot v_{i+1}\cdots v_{\ell_y}\) where:
\begin{align*}
  \ell_x \ud
  \begin{cases}
    |w|     & \!\!\!\text{if } i {+} |w| {-} j \leq k\\
    k {-} i & \!\!\!\text{otherwise}
  \end{cases} & \quad &
  \ell_y \ud
  \begin{cases}
    |v|     & \!\!\!\text{if } j {+} |v| {-} i \leq k\\
    k {-} j & \!\!\!\text{otherwise}
  \end{cases}
\enspace .
\end{align*}

\subsection{Replacement}
\label{subsection:replacement}
We follow an \emph{elitist} replacement.
We select a \emph{children rate}, \(\chr \in [0,1]\), which represents the proportion of individuals in the new population that will be children from the crossover step.
This means that we select from the population of children a total of \(C = \chr \cdot N\) individuals.
We perform this selection according to the same procedure as the one used in the selection step, except that we also avoid selecting twice the same index in order to extend the search space explored.
This way, we decrease the number of identical individuals in the new population.
We select the remaining \(N - C\) individuals for the new population from the original population (after the mutation step below).
Again, these are selected in the same way as the children, i.e., using the selection step and avoiding the choice of the same index twice.

\subsection{Mutation}
In parallel to the crossover step, we apply the mutation step.
We select the \emph{mutation probability}, \(\mutp\in [0,1]\), which represents the probability of an individual from the original population to go through the mutation procedure.
This procedure is called \emph{single-point} mutation and an efficient way to implement it is as follows.
We fix \(\lambda \in [0,1]\) which represents the probability of mutating each single symbol in a given individual \(w \in \Sigma^{\leq k}\).
Then we generate sequentially random numbers using an exponential distribution with parameter \(\lambda\) which will operate as the positions of \(w\) that will be mutated.
We stop this number generation when the last position computed is equal or greater than \(|w|\).
At each position generated we perform one out of 3 possible mutations chosen uniformly at random: deletion, insertion or replacement (by a different symbol in the alphabet) of a symbol.

Additionally, we include a way to introduce extra variability in the population by keeping track of the number of times the best individual in the population has repeated along different generations.
If this number exceeds a threshold we fix to \(10\) repetitions then we triple the probability \(\mutp\) during the next generation.

The population that results after the mutation procedure is the so-called \emph{mutant population}.

\subsection{Evaluation}
\label{subsection:evaluation}
This step calculates the fitness function of each individual.
We use it after the initialization step and during the selection procedure, as the latter relies on the fitness function evaluation of each individual.

\subsection{Termination condition}
Our \emph{termination condition} depends on the execution time (we set a timeout of $T (>0)$ seconds).
If the execution time reaches $T$ our algorithm halts.

\subsection{Fitness function: Memoization}
\label{section:f-memo}

We define the fitness function evaluated on each individual as its weight, which is computed as the matrix product described in Equation~\eqref{eq:weight_word}.
In consequence, matrix multiplication is the most time-consuming operation of the algorithm.\footnote{Consuming \(80-95\%\) of the total run-time if no optimization is performed.}
On the other hand, matrix associativity enables the \emph{memoization} of partial products allowing for time-efficiency optimizations.
We implement a simple \emph{memoization} technique by means of a lookup table as follows.

The lookup table is a hash table that is precomputed before the initialization step of the genetic algorithm and remains invariant throughout the execution.
The keys of the table are all the words of length at most \(B( > 1)\), where \(B\) is an adequate value that takes into account the size of the alphabet, and balances the time and space cost of initializing the table and the time gain in the computation of the weights along the execution.
We refer to this value as the \emph{(maximum) block size}.
The value associated to each key \(w = a_1\cdots a_n\) is the matrix \(\prod_{i=1}^{n} M_{a_i} \in \rationals^{|Q| \times |Q|}\), where \(Q\) and \(M_{a_i}\) are the states and transition matrices of the input WFA, respectively.

This way, to compute the weight of \(w\), we decompose the word into \(\floor{\frac{n}{B}} + 1\) subwords where the first \(\floor{\frac{n}{B}}\) of them are of size \(B\) and the last one is of size \(n \mod B\).
Then, the weight of each block is retrieved from the table to compute the weight of \(w\) as in Equation~\eqref{eq:weight_word}.
To minimize the number of operations (products of rationals) in the latter equation, we first compute the vector-matrix product of the initial vector and the matrix corresponding to the first block of the word, and then the result is multiplied by the matrix for the second block, and so on.

In Section~\ref{subsection:memoization}, we show how this simple memoization technique on WFAs improves the time efficiency of the algorithm.

\subsection{Implementation}
We implemented our algorithm in C.
The procedure takes as input the matrix representation of a WFA \(\cA\) and a value \(k \geq 1\).
For the code representation of the weights we use rationals of arbitrary precision by means of the library for arbitrary-precision arithmetic on rational numbers, GMP~\cite{GMP}.
For the lookup table implementation we use  \texttt{uthash.h}~\cite{uthash} which is a header file written
in C that implements a hash table for handling C structures.

Our source code is publicly available and open source for reproducibility and verifiability.\footnote{At GitHub: \href{https://github.com/elenagutiv/ga-wfas}{https://github.com/elenagutiv/ga-wfas}}

\section{Experimental Results}
\label{section:experimental}
We conducted experiments to evaluate the performance of our algorithm.
In particular, we tackle the two following questions:

\begin{enumerate}[leftmargin = 5ex]
  \item Is our genetic-algorithm-based metaheuristic a good fit for approximating the BWMP?
  \item Given that memoization of partial executions is an appropriate optimization for WFAs, what is an lower bound on the gain of using this technique?
\end{enumerate}

Regarding question \((1)\), we compare our algorithm with random search to evaluate the adequacy of our problem to a genetic-based technique (Section~\ref{subsection:comparison}).
Additionally, we evaluate the quality of the solutions found by our algorithm by designing an experiment where the word with the highest weight is known (Section~\ref{subsection:bf-search}).

For question \((2)\), we compare the performance of our algorithm with a version of it where no memoization is performed in order to establish an lower bound on the gain of using this technique (Section~~\ref{subsection:memoization}).

\paragraph{Experimental Setting}
To carry out the experiments we built a set of benchmarks composed of 12 WFAs which we call \emph{Random}.
Each of these WFAs was obtained as the result of executing the WFA extraction procedure developed by Okudono et al.~\cite{tmp-Takamasa2019}.
In each case, the input of the extraction procedure was an RNN trained with a set of input-output pairs \((w, W_{\cA}(w))\)
 where \(w \in \Sigma^{\leq 20}\) and \(\cA = (Q, \Sigma, \{M_a\}_{a\in \Sigma}, \initial, \final)\) was a WFA randomly generated with \(|Q| \in \{10, \ldots, 20\}\), \(|\Sigma| \in \{4, 6, 10\}\) and \(W_{\cA}: \Sigma^{\leq 20} \rightarrow [0,1]\cap \rat\).
The size of the state spaces of each extracted WFA \(\cA' = (Q', \Sigma, \{M'_a\}_{a\in \Sigma}, \initial', \final')\) in \emph{Random} is between \(6\) and \(25\) states and, obviously, the size of their alphabets remains in \(\{4,6,10\}\).
We will name each of the 12 extracted WFAs \(\cA'\) as \(X(|\Sigma|, |Q'|)\), where \(X \in \{A, B, \ldots, L\}\).

Table~\ref{tab:parameter} shows the values of the parameters we fix for the experiments.
The termination condition only depends on the timeout \(T\) which is set to \(120\) seconds for all the experiments performed (except from the experiment in Section~\ref{subsection:bf-search}).
This value is chosen to obtain significant results in an admissible amount of time.
For the rest of the parameters we empirically selected the most appropriate values for a good performance.
In particular, in order to choose an appropriate value for the maximum block size \(B\) of the hash table for each benchmark, we consider the size of the alphabet and the number of states of each WFA in order to keep constant the time of hash initialization.
Thus, \(B\) is \(7, 6\) and \(5\) for alphabet sizes of \(4, 6\) and \(10\) respectively when the number of states of the corresponding WFA is less or equal to \(12\) , while we decrease the value \(B\) in 1 unit if the number of states is greater than \(12\).

We run our experiments on a Debian/GNU Linux 9.0 machine of 64 bits with 72 virtual cores (Xeon Gold 6154 @3GHz) and a RAM of 64 GB (DDR4 @2666 MHz).
\begin{table}[!htbp]
  \caption{Parameter Values}
  \label{tab:parameter}
  \begin{tabular}{cccccccl}
    \toprule
    \(k\) & \(N\) & \(\beta\) & \emph{cr} & \emph{mp} & \(\lambda\) & \(T\)\\
    \midrule
    20 & 200 & 30 & 0.8 & 0.1 & 0.1 & 120 (s)\\
    \bottomrule
  \end{tabular}
\end{table}

\subsection{Comparison with random search}
\label{subsection:comparison}
The goal of this experiment is to determine how well a genetic-algorithm-based solution fits the BWMP.
To do so we compare the performance of our metaheuristic against a \emph{random search} algorithm, an unspecialized search method that does not assume any inner structure on the problem.

Random search generates at each iteration a  word \(w \in \Sigma^{\leq 20}\) uniformly at random\footnote{We first choose a length value in the interval \([1,20]\) and then, at each position of the word, we choose a symbol in the alphabet. Both selections are performed uniformly at random.} and computes its weight w.r.t. input WFA.
Initially, it stores the first weight observed and, at each iteration, updates this value if a higher weight is found.
Finally, it returns the last weight stored.

In Table~\ref{tab:comparison} we say that a weight is \emph{observed} if{}f a word with that weight is analyzed by the algorithm (we only consider each word once, even if they are analyzed twice or more times).
This way, the columns \emph{Weights observed} show the average of all the weights observed by each algorithm after \(10\) executions, with a \(95\%\) confidence interval (in gray).
In bold and between parentheses, we show the maximum weight found by each procedure.
Finally, columns w/s show the \emph{total number words analyzed per second} by each algorithm.

We also attach 4 histograms in Figure~\ref{fig:histograms}, each corresponding to the \emph{observed-weight distributions} of one single execution of the random search and our metaheuristic for the same input WFA, in order to illustrate and complement the information given in Table~\ref{tab:comparison}.
We add to each histogram a colored vertical line indicating the maximum weight found by each algorithm.

First, notice that the random search weight distribution estimates the actual weight distribution given by the input WFA (this estimation improves as we increase the execution time).
Thus, the average shown in the column \emph{Weights observed} of random search (Table~\ref{tab:comparison}) is an estimation of the average weight of the words read by the input WFA.
The average of the weights observed by our algorithm is always greater than that of random search.
So is the maximum found, as well.
This means, that our algorithm is able to reach infrequent weight values as long as they are better than those already seen.

All these observations are illustrated in Figure~\ref{fig:histograms}.
Notice that, while the observed-weight distribution of random search estimates the actual weight distribution of the input WFA, our genetic algorithm tends to reach more frequently better solutions.

On the other hand, random search analyzes in average \(10.4\) times more words per second than our algorithm (Table~\ref{tab:comparison})
This is expected as the genetic-algorithm machinery is computationally heavier.

We conclude that our genetic-based metaheuristic is able to exploit the internal structure of WFAs to outperform a black-box method such as random search.
These results enhance the use of this method as an alternative to the lack of specialized algorithms in scenarios such as WFAs that result from automata learning techniques~\cite{Balle2015}.

\begin{table}[t]
  \caption{Observed-weight distributions and number of the words analyzed per second (w/s) by random search and our algorithm.
  The columns \emph{Weights observed} show the average of all the weights observed by each algorithm, with a \(95\%\) confidence interval (in gray). In bold and between parentheses, we show the maximum weight found by each procedure.
  The columns \emph{w/s} show the total number words analyzed per second by each algorithm.}
  \label{tab:comparison}
  \begin{tabular}{lll|ll}
    \toprule
    & \multicolumn{2}{c}{Random search}      & \multicolumn{2}{c}{Genetic  algorithm}\\
    \midrule
    & Weights observed & w/s & Weights observed & w/s\\
    \midrule \midrule
    A(4,7) & 0.77  \conf{0.19}  \maxi{1.02} & 14.21k & 0.96  \conf{0.13}  \maxi{1.07} & 1.5k\\
    B(4,11) & 0.78  \conf{0.17}  \maxi{0.98} & 6.22k & 0.93  \conf{0.12}  \maxi{1.03} & 0.76k\\
    C(4,12) & 0.5  \conf{0.15}  \maxi{0.71} & 5.77k & 0.59  \conf{0.2} \maxi{0.73} & 0.68k\\
    D(4,15) & 0.4  \conf{0.24}  \maxi{0.73} & 3.55k & 0.67  \conf{0.23}  \maxi{0.84} & 0.14k\\
    E(6,9) & 0.67  \conf{0.2}  \maxi{1.04} & 8.65k & 0.92  \conf{0.22}  \maxi{1.17} & 0.95k\\
    F(6,12) & 0.46  \conf{0.15}  \maxi{0.74} & 5.24k & 0.63  \conf{0.22}  \maxi{0.86} & 0.62k\\
    G(6,13) & 0.67  \conf{0.18}  \maxi{0.96} & 4.14k & 0.87  \conf{0.2}  \maxi{1.04} & 0.49k\\
    H(6,25) & 0.43  \conf{0.23}  \maxi{0.79} & 1.21k & 0.65  \conf{0.31}  \maxi{0.86} & 0.14k\\
    I(10,6) & 0.47  \conf{0.22}  \maxi{0.78} & 14.3k & 0.76  \conf{0.17}  \maxi{0.92} & 1.42k\\
    J(10,6) & 0.5  \conf{0.23}  \maxi{0.82} & 13.38k & 0.73  \conf{0.28}  \maxi{0.91} & 1.44k\\
    K(10,6) & 0.43  \conf{0.16}  \maxi{0.66} & 14.26k & 0.61  \conf{0.19}  \maxi{0.77} & 1.5k\\
    L(10,6) & 0.38  \conf{0.22}  \maxi{0.66} & 14.3k & 0.55  \conf{0.3} \maxi{0.76} & 1.51k\\
    \bottomrule
  \end{tabular}
\end{table}

\begin{figure}[!htbp]
  \caption{Observed-weight distributions of random search (\emph{Random}) and our algorithm (\emph{Genetic}).
  The maximum weight found by each algorithm is indicated by a colored vertical line.}
  \label{fig:histograms}
  \begin{minipage}{.5\textwidth}
    \begin{minipage}{.5\textwidth}
      \centering
      \includegraphics[width=\linewidth]{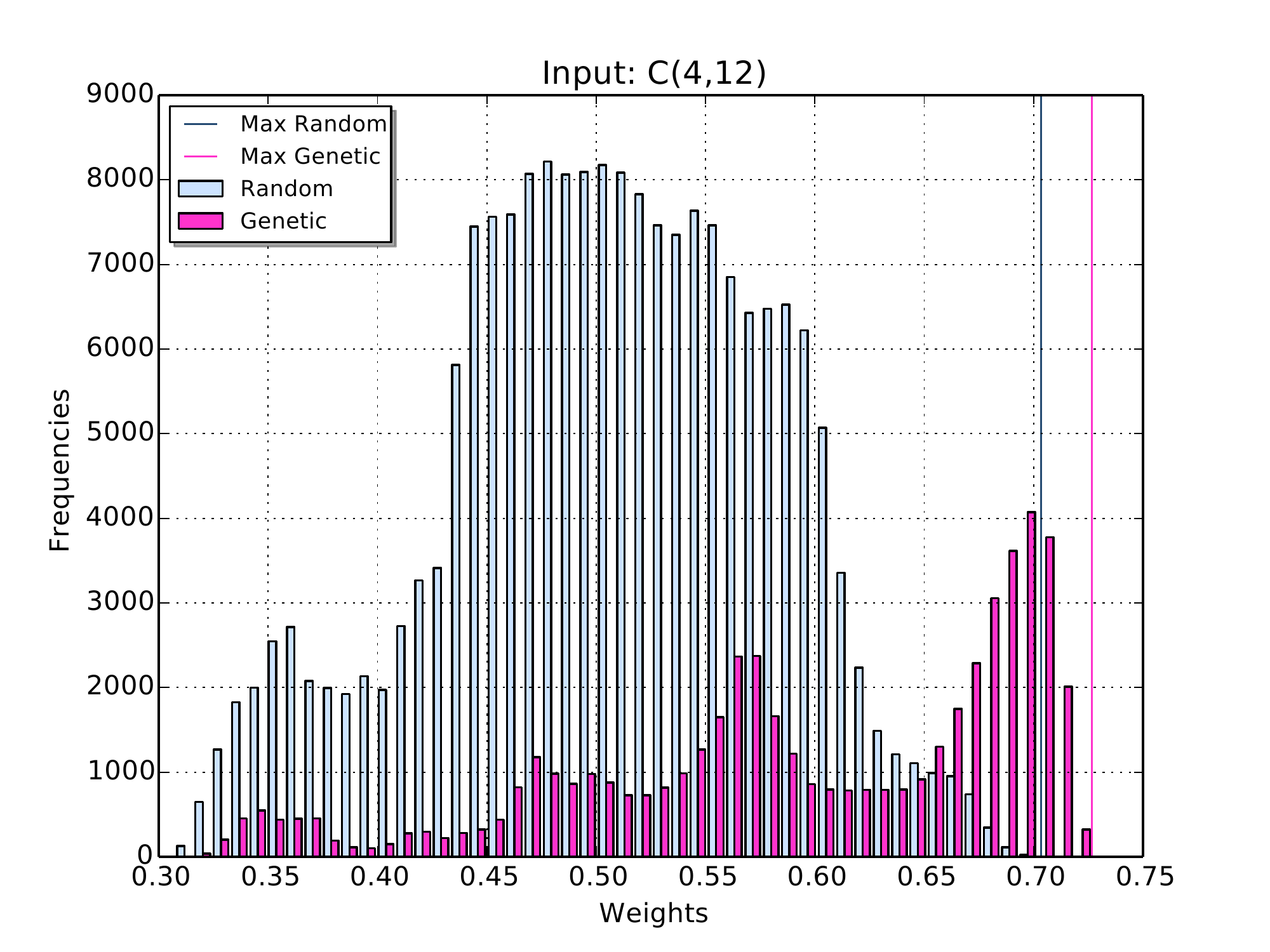}
    \end{minipage}
    \begin{minipage}{.5\textwidth}
      \centering
      \includegraphics[width=\linewidth]{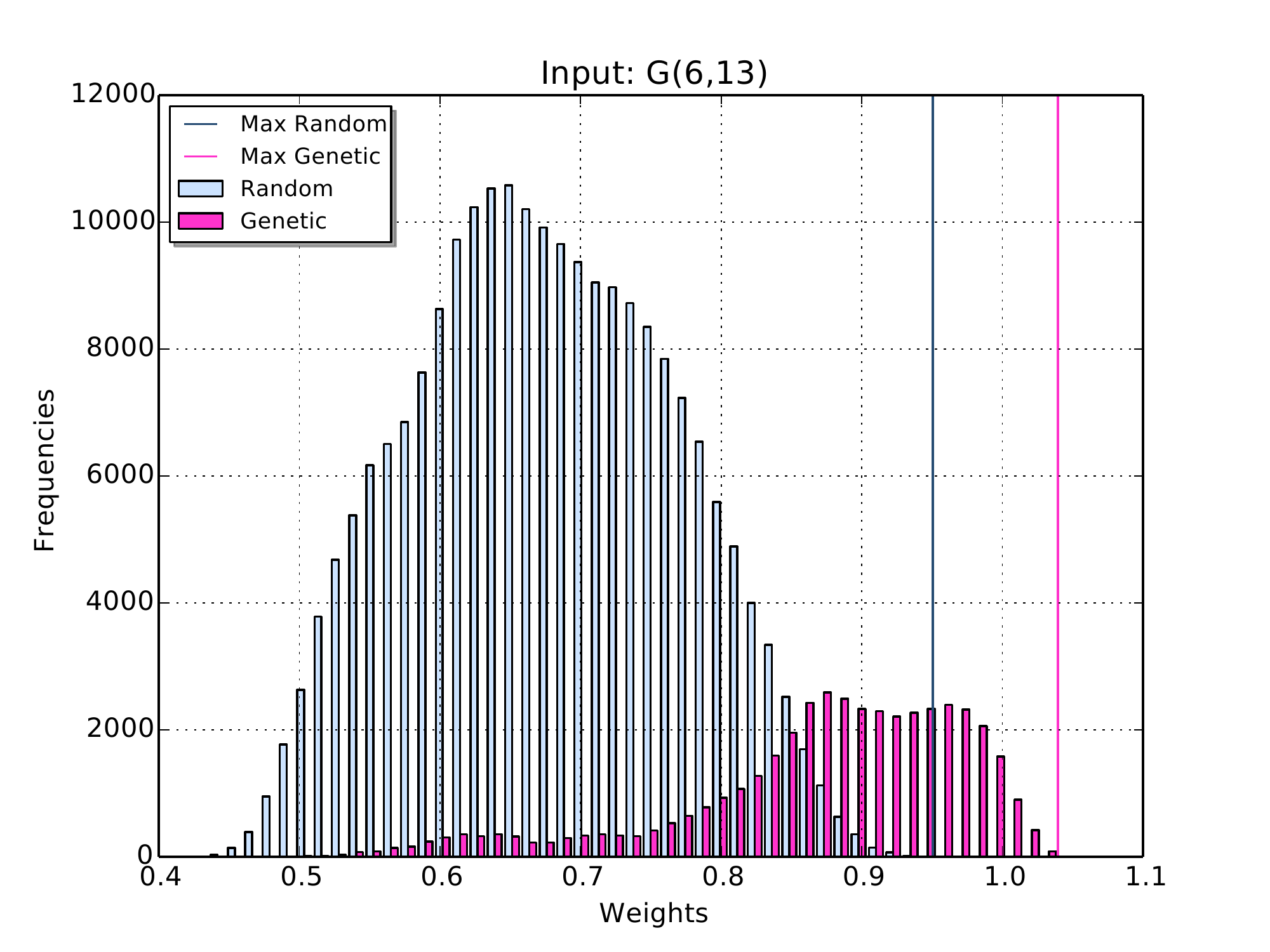}
    \end{minipage}
  \end{minipage}
  \begin{minipage}{.5\textwidth}
    \begin{minipage}{.5\textwidth}
      \centering
      \includegraphics[width=\linewidth]{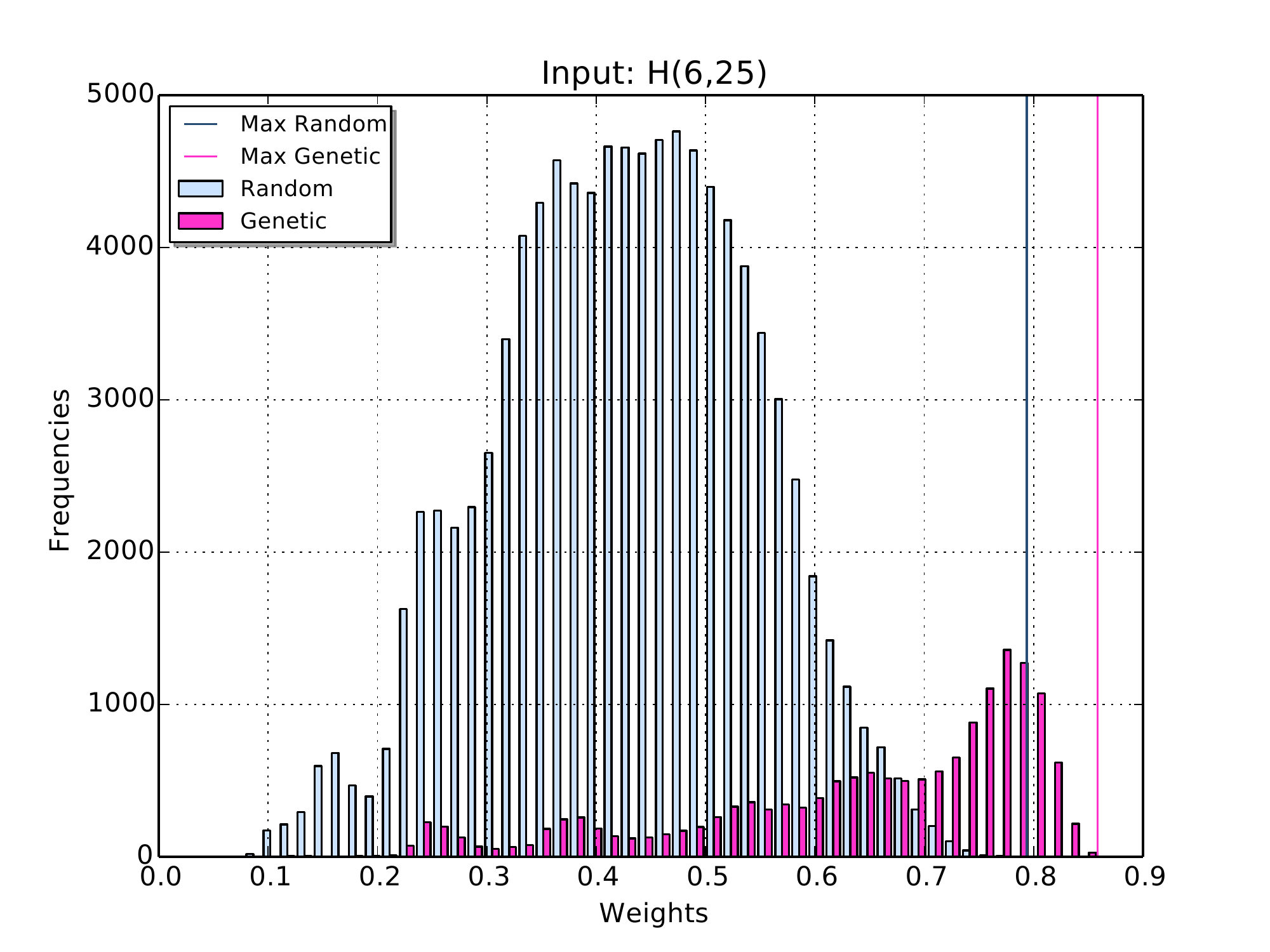}
    \end{minipage}
    \begin{minipage}{.5\textwidth}
      \centering
      \includegraphics[width=\linewidth]{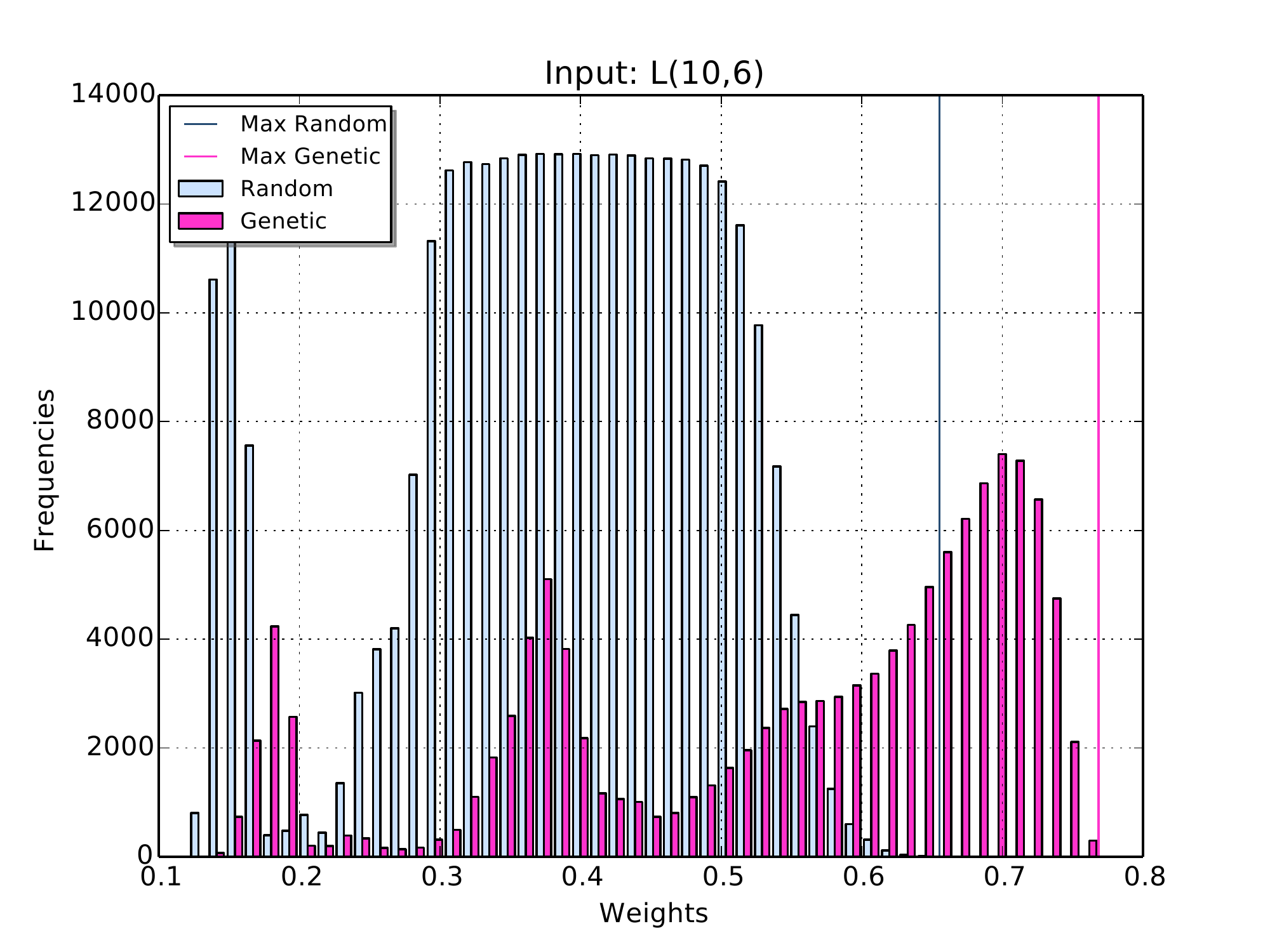}
    \end{minipage}
  \end{minipage}
\end{figure}

\subsection{Examples with known maximum weight}
\label{subsection:bf-search}
In this section, we evaluate the quality of the solutions found by our algorithm when the word with the highest weight in the automaton is known.
To do so, fixed an appropriate bound \(k\), we will run our genetic algorithm and compare the maximum weight found with the actual maximum weight over all the words of length at most \(k\).

Specifically, we will use a subset of \(9\) benchmarks of Random, and we will perform exhaustive search to solve the BWMP exactly.
The length bound \(k\) of the problem is fixed to a value that makes feasible this search in a reasonable amount of time.
In particular, the choice of \(k\) for each benchmark depends on the size of the alphabet of the WFA, and thus, for an alphabet size of \(4,6\) and \(10\), the value of \(k\) we choose is \(14, 11\) and \(9\), respectively.
Finally, the subset of \(9\) benchmarks corresponds to those WFAs in Random with at most \(13\) states.
This way, we also avoid the cases in which exhaustive search would need an excessive amount of time to finish.

On the other hand, we run our genetic algorithm on the same subset of \(9\) benchmarks, executing the procedure \(10\) times on each WFA, with a fixed timeout\footnote{We update the timeout value from 120 to 240 seconds in order to give sufficient time to our algorithm to obtain a fair comparison w.r.t. results obtained by the exhaustive search.} of \(240\) seconds.

Table~\ref{tab:comparison-bf} compares the maximum weights found by exhaustive search and our algorithm.
We also include in the table the information about the position of the average maximum weight found by our algorithm w.r.t. the list of \(n\) best weights output by exhaustive search.
In particular, if the position is \(1\), then the maximum weight found by our genetic algorithm always coincides with the highest weight of length at most \(k\) in the given WFA.

Finally, we notice that in the cases where our algorithm performs worse, and specially in the case of \(L(10,6)\), the WFA contains local maxima that corresponds to words significantly different to the word of highest weight.
These cases are particularly difficult for the genetic search, and strategies like replacing part of the population by new random individuals might solve premature convergence.

\begin{table}[t]
  \caption{Maximum weights found by exhaustive search and our algorithm. Note that the column \emph{Max. found} of \emph{Exhaustive search} correspond to the actual highest weight in the WFA of length less or equal to \(k\), where \(k\) is \(14,11\) and \(9\) when the alphabet size is of \(4,6\) and \(10\) respectively. The column \emph{Max. found} of \emph{Genetic algorithm} is the average of the maximum weights found in the \(10\) executions. The column \emph{Position w.r.t. the maximum} indicates the position of the maximum weight found by our GA in the list of \(n\) highest weights output by exhaustive search.
  We indicate with a superscript \(\dagger\) those cases where our algorithm found the maximum weight at least once.}
  \label{tab:comparison-bf}
  \begin{tabular}{ll|ll}
    \toprule
    & \multicolumn{1}{c}{Exhaustive search}      & \multicolumn{2}{c}{Genetic  algorithm}\\
    \midrule
    & Max. found & Max. found & Position w.r.t. \\
    & & & the maximum\\
    \midrule \midrule
    A(4,7) & 1.0297 & 1.0286 & \(2^{\dagger}\)\\
    B(4,11) & 0.9895 & 0.9895 & \(1^{\dagger}\)\\
    C(4,12) & 0.7099 & 0.7037 & 5\\
    E(6,9) & 0.9898 & 0.9850 & 3\\
    F(6,12) & 0.7281 & 0.7128 & 4\\
    I(10,6) & 0.7631 & 0.7631 & \(1^{\dagger}\)\\
    J(10,6) & 0.7723 & 0.7711 & \(2^{\dagger}\)\\
    K(10,6) & 0.6150 & 0.6120 & \(7^{\dagger}\)\\
    L(10,6) & 0.6442 & 0.5982 & 16\\
    \bottomrule
  \end{tabular}
\end{table}

\subsection{Memoization}
\label{subsection:memoization}

The goal of this experiment is to establish an lower bound on the gain of using a simple memoization technique (see Section~\ref{section:f-memo}) in our algorithm.
Specifically, we perform a comparison of the ratio of \emph{words analyzed per second} on two versions of our algorithm: one that does not use a look-up table when computing the fitness function of each individual, and another that uses it.
We call these versions the \emph{no-memoization} and the \emph{memoization} version, respectively.

In Table~\ref{tab:memoization}, each row collects the \emph{average} of the values measured for each benchmark from the set \emph{Random} obtained from its \(10\) executions.
Each column "words analyzed/s" corresponds to the \emph{total number of words analyzed per second} by each version.

We conclude that, in average, this simple memoization technique achieves a number of words analyzed \(3.8\) times greater than that of the non-optimized version.
Indeed, this technique can be further optimized.
The number of words analyzed by the algorithm is too large to perform an efficient on-the-fly memoization, where every word that has not been observed before is stored in the hash table.
However, it would be interesting to explore efficient ways to store the most frequent prefixes or suffixes, for instance.

\begin{table}[t]
  \caption{Comparison of the no-memoization vs. the memoization version of our algorithm in terms of words analyzed per second (w/s).}
 \label{tab:memoization}
 \small
  \begin{tabular}{lc|c}
    \toprule
      & {No-Memo}  & {Memoization}\\
      \midrule
      &  w/s & w/s\\
    \midrule \midrule
    A(4,7) & 0.4k &  1.5k\\
    B(4,11) & 0.18k &  0.76k\\
    C(4,12) & 0.15k &  0.68k\\
    D(4,15) & 0.04k &  0.17k\\
    E(6,9) & 0.25k &  0.95k\\
    F(6,12) & 0.14k &  0.61k\\
    G(6,13) & 0.12k &  0.49k\\
    H(6,25) & 0.03k &  0.14k\\
    I(10,6) & 0.48k &  1.44k\\
    J(10,6) & 0.48k &  1.44k\\
    K(10,6) & 0.47k &  1.44k\\
    L(10,6) & 0.48k &  1.51k\\
    \bottomrule\\[-0.5cm]
  \end{tabular}
\end{table}

\subsection{Other observations}
\label{section:observations}

In order to improve our algorithm we implemented several variants.
For the sake of space, we describe two of the most remarkable ones:
\begin{itemize}[leftmargin = 4ex] \setlength{\itemsep}{1ex}
  \item We implemented an alternative initialization method to random initialization that computes the \(N\)-best individuals from the set of words \(\Sigma^{\leq B}\), where \(B \geq 1\) is the maximum block size in the look-up table.
  The idea was to start with a population of short but higher-weighted words in the WFA and \emph{extend} them, i.e., combine and mutate them, with the hope that those good short patterns found in the WFA could be repeated and combined to generate higher weighted words.
  We exploited the look-up table to obtain the weights of the words in \(\Sigma^{\leq B}\) efficiently.
  By analyzing the observed-weight distribution of both techniques we concluded that no significant gain was achieved w.r.t. the simpler random initialization, which evidences the ability of our algorithm to escape from local optima and explore wider regions of the search space.
  \item We introduced a variant in the crossover procedure: instead of choosing just one single pair of indices, we computed a set \(\cI\) of them as well as the weights of each resulting pair of children.
  Then, we selected the \emph{best} one according to either the average weight or the maximum weight of the two children generated.
  We performed experiments comparing the performance when the size of \(\cI\) ranged from 1 to 10 pairs of indices at each crossover operation. 
  We could not observe any relevant improvement in the solutions obtained  when \(|\cI| > 1\) that justified the extra computational cost of this technique.

\end{itemize}
\section{Case of Study}
\label{section:applications}
In this section we perform a case study to illustrate how our algorithm can be used for the
light-weight verification of an RNN against a weighted regular specification.
We will use our genetic algorithm in combination with the procedure that extracts a WFA
from a given RNN~\cite{tmp-Takamasa2019}, to estimate the error between the extracted WFA and the WFA that describes the specification of the RNN over a bounded-length set of words.
In turn, we will obtain an estimation of the error together with an evidence of why
the network is not properly approximating its specification.

First, we define a notion of \emph{distance} between two WFAs over a length-bounded set of words.
Recall that given two WFAs \(\cA\) and \(\cB\), it is possible to construct the WFA denoted by \(\cA\ominus \cB\) such that \((\cA\ominus \cB)(w) = \cA(w) - \cB(w)\), for all \(w \in \Sigma^*\).

\begin{definition}
\label{def:distance}
Given two WFAs \(\cA\) and \(\cB\) and a natural \(k \geq 1\), we define the
\emph{distance between \(\cA\) and \(\cB\) over \(\Sigma^{\leq k}\)} as

\[
  d^k(\cA, \cB) \ud
  \max_{w\in\Sigma^{\leq k}} \{ (\cA \ominus \cB)(w), (\cB \ominus \cA)(w)\}
  \enspace .
\]
\end{definition}

Note that this corresponds to defining the distance between the two automata as the difference
(in absolute value) of the weights in which they differ the most.

Intuitively, given a WFA \(\cA\) that abstracts the behavior of a system, and a WFA \(\cB\) that defines its specification, this notion of distance describes the maximum ``error" of \(\cA\) approximating \(\cB\) on a length-bounded set of words.
Thus, using our algorithm to give an approximation of the distance will provide an estimation of this error.
Furthermore, the words on which the maximum is reached
are witnesses of the behavioral difference between the system and the specification.

\newcommand{\s}[1]{\text{\sf {\small #1}}}
For the sake of our case study, we consider a problem that (as we argue below) admits as simple specification
as a WFAs. Namely, the problem of deciding whether a numerical expression contains well-balanced parenthesis.
For this, we trained an RNN with a set of 9000 input-output pairs \((w, f(w))\) as in~\cite{tmp-Takamasa2019}
where the 12 symbols alphabet is defined as \(\Sigma \ud \{\s{0}, \ldots, \s{9}, \s{(}, \s{)}\}\),
\(w \in \Sigma^{\leq 20}\) and the function \(f: \Sigma^{\leq 20} \rightarrow [0,1]\) is defined as:
\[f(w) \ud
\begin{cases}
1 - 2^{-\Delta(w)} & \text{if \(w\) is well-parenthesized and }\Delta(w) \leq 2\\
0 & \text{otherwise}
\end{cases} \enspace ,\]
where \(\Delta(w)\) represents the depth of the deepest balanced pair of parentheses in \(w\)
For instance,
\begin{align*}
  \s{``(1)(2)"}   \fmapsto 1/2 &&
  \s{``((1))(2)"} \fmapsto 3/4 &&
  \s{``((1(2)"}   \fmapsto 0 &&
  \s{``(((1)))"}  \fmapsto 0
\end{align*}
Observe that the function $f$ assigns weight $0$ to unbalanced words or balanced words with depth greater than $2$ and a positive weight otherwise.
Note that the weighted language described by this function is regular, thus it can be modeled by means of a WFA \(\cA_{E}\) (a construction with \(8 \) states is deferred to Section~\ref{app:construction} in the Appendix.
The extracted WFA \(\cA_R\) has \(6\) states and thus, the WFAs \(\cA_R \ominus \cA_E\) and \(\cA_E \ominus \cA_R\) have \(14\) states each.
We run our genetic algorithm on both WFAs (with the parameter values described in Table~\ref{tab:parameter})
and a timeout of $T = 240s$ instead, to increase the accuracy of our results.
In Figure~\ref{fig:distance} we represent the observed-weights distributions
of \(\cA_E \ominus \cA_R\) and \(\cA_R \ominus \cA_E\) in one of the execution of our algorithm.

We obtain that the maximum error between \(\cA_R\) and the specification \(\cA_E\) is approximately \( 0.0141\).
On the other hand, the 5 words with the highest error and its corresponding error are:
\[
  \begin{array}{ccccc}
    \s{``()777999()"} & \s{``()779999()"} & \s{``()797974()"} & \s{``()749909()"} & \s{``()999999()"} \\
        0.01410   &    0.01405    &     0.01402   &     0.01399   &     0.01397
  \end{array}
\]

We notice that the maximum weight found for input \(\cA_R \ominus \cA_E\) is almost \(0\), which means that
the automaton trained on the RNN (and thus the RNN itself\,\footnote{The procedure that extracts a WFA
  from a RNN can provide an accuracy, computed as the mean squared error, that bounds the maximum error
  between the WFA and the RNN~\cite{tmp-Takamasa2019}.}) tends to under approximate the specification values.
The histogram also provides information on the number of distinct words that are witness
of the behavioral difference between the automata.

Our analysis illustrates the differences between the trained neural network and the specification.
For instance, if the RNN designer had established an error tolerance lower than $0.013$, they would need
to revisit the RNN training process. Furthermore, the misclassified inputs with highest error, e.g.,
\s{``()777999()''}, \s{``()779999()''} or \s{``()797974()''} can be used as a hint on
the process of revision.

\begin{figure}[t]
  \captionof{figure}{Observed-weight distribution of \(\cA_E \ominus \cA_R\) and \mbox{\(\cA_R \ominus \cA_E\).}}
  \centering
  \includegraphics[width=\linewidth]{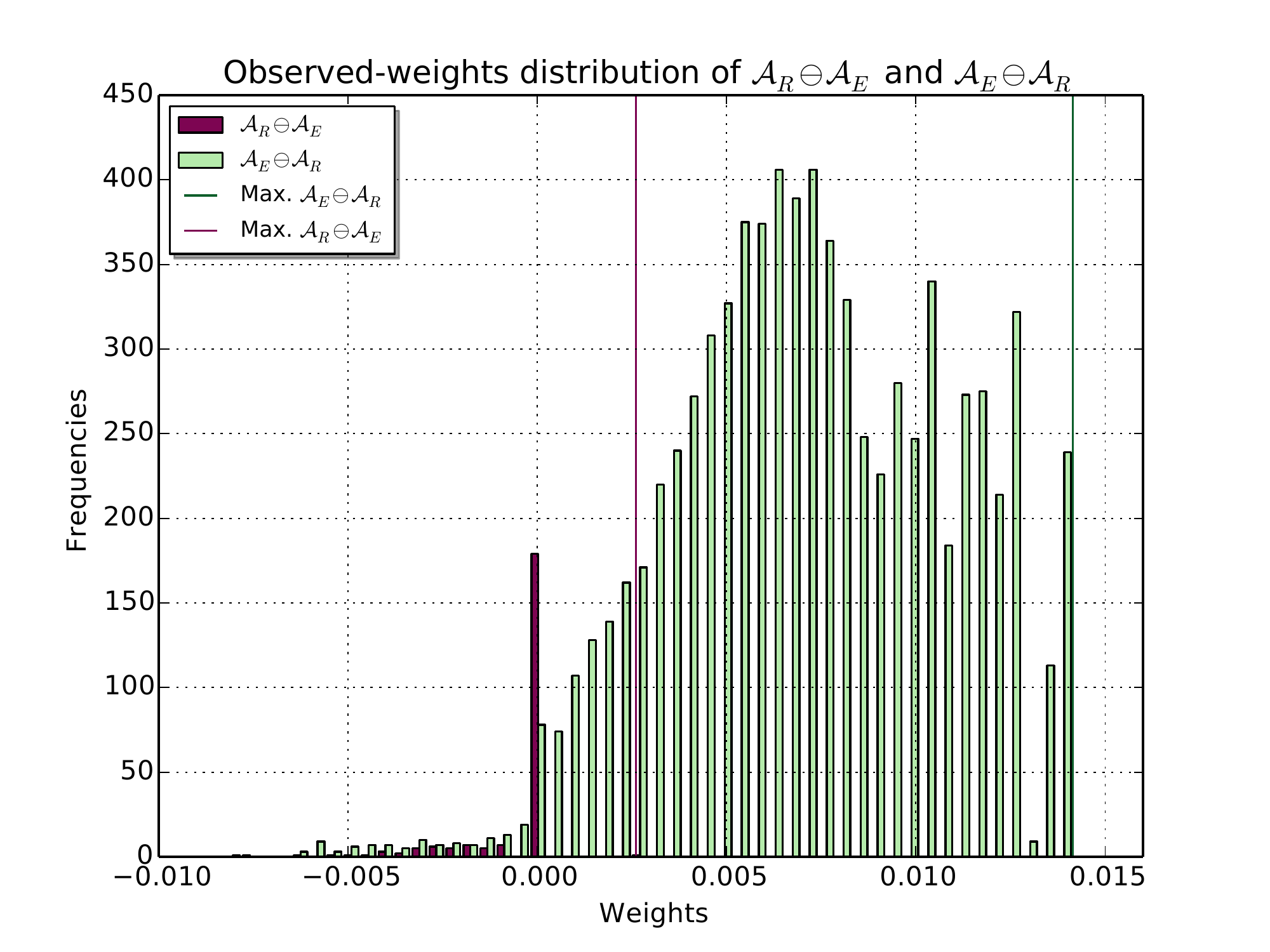}
  \label{fig:distance}
\end{figure}

\section{Conclusions and Future Work}
We study the problem of finding the word with the highest weight in WFAs with weights over \(\rationals\).
We propose a metaheuristic based on the genetic algorithm to approximate solutions to the BWMP, showing that this method benefits from optimization techniques such as memoization of partial executions, and outperforms black-box methods such as random search, which enhances the use of the algorithm in cases where specialized algorithms for the BWMP are not available, in particular, WFAs resulting from automata learning techniques~\cite{Balle2015}.
Regarding the latter scenario, we use our method in combination with a recent procedure that extracts a WFA from an RNN~\cite{tmp-Takamasa2019} to show its potential for detecting misclassified input words in the RNN as well as estimate the maximum error between the network and its target function.

One line of future work is to explore other crossover operators, being \emph{k-point crossover} of particular interest as it may further exploit the structure of cycles that is proper of general WFAs.
We conjecture that this technique may be particularly effective on WFAs with a large number of states.
Additionally, it would be interesting to explore how effective our algorithm is on real applications.
emerging from other contexts.
For instance, discrete-linear systems can be modeled as WFAs under the assumption that their input function takes a finite number of possible values.
Our techniques could be used to compare two different systems (controller and plant) with expected similar behavior.

Finally, we extend previous results by proving that the bounded decisional version is NP-complete
even in the case of integer weights and two alphabet symbols.

\section{Appendix}
\subsection{The BTRP is NP-complete}
\label{app:np-completeness}

In order to show that the \emph{Bounded Threshold Reachability Problem} (Definition~\ref{def:BTRP}) is NP-complete,
we first prove that so is the \emph{Bounded Equality Reachability Problem}.

\begin{definition}[Bounded Equality Reachability Problem]
  Given a WFA \(\cA\) over the rational numbers, \(k \in \nat\) and \(\nu \in \rationals\),
  the \emph{Bounded Equality Reachability Problem} (BERP) consists of deciding whether there exists
  a word \(w \in \Sigma^{k}\) such that \(W_{\cA}(w) = \nu\).
\end{definition}

Our proof of this result relies on a reduction from the well-known NP-complete \emph{Hamiltonian cycle} problem, which
we define here for completeness.
A \emph{simple directed graph} is a tuple \(G = (V, E)\) where \(V\) is a set of elements called
\emph{vertices} and $E$ is a subset of ordered pairs, \(E \subseteq V \times V\), called \emph{directed edges}.
Given an edge $e \coloneqq (u,v)$, we denote its \emph{source} $u$ by \(\source(e) = u\)
and its \emph{destination} $v$ by $\dest(e)$.
A \emph{cycle} in \(G\) is a finite sequence of edges \(e_1 \cdots e_n\) satisfying
\(\dest(e_i) = \source(e_{i+1})\) for all \(i \in [n{-}1]\) and $d(e_n) = s(e_1)$.
A cycle is said to be \emph{Hamiltonian} if it visites every vertex exactly once, that is,
$n = |V|$ and for every $v \in V$, $\exists j \in [n]$ such that $v = \dest(e_j)$.

\begin{definition}[The Hamiltonian Cycle Problem]
  Given a simple directed graph \(G = (V,E)\), the \emph{Hamiltonian Cycle Problem} (HCP) consists of
  deciding whether there exist a Hamiltonian cycle in \(G\).
\end{definition}

We use the fact that the HCP is NP-complete (Karp~\cite{Karp72}) to prove the following theorem, which
holds even if the alphabet is restricted to contain only two symbols.

\begin{theorem}
\label{thm:BERP}
The BERP with weights in $\nat$ is NP-complete.
\end{theorem}

\begin{proof}
  We first prove that the BERP is in NP. For this, it is essential to properly define the size of a problem.
  Let an instance $x$ of the BERP be given by \(x \coloneqq (1^k, \cA, \alpha)\) where $k \in \nat$,
  \(\cA\) is a WFA over the rational numbers and $\alpha \in \rationals$.
  We define the input size of $x$ as $|x| \coloneqq k + |\cA| + |\alpha|$, where $|\cA|$ and $|\alpha|$
  are the number of bits used to represent $\cA$ and $\alpha$ respectively.
  In these conditions, we can prove that the BERP is in NP by arguing that yes-instances can be
  efficiently verified.
  More formally, there exists a deterministic Turing machine $M$, running in polynomial time on the
  size of its first input and there exists a polynomial $p$, such that for all yes-instances $x$, there exists
  $y$, with $|y| < p(|x|)$ and such that $M(x,y) = 1$;
  and for every no-instance $x$ and all $y$ with $|y| < p(|x|)$, $M(x,y) = 0$.
  Here, $M$ can be the algorithm that parses $y$ as a word, computes its weight in $\cA$ and returns $1$
  iff the weight equals $\alpha$ and the number of alphabet symbols in $y$ is upper-bounded by $k$.

  In order to show that the BERP is NP-complete, we provide a reduction from the HCP.
  More concretely, let \(G = (V,E)\) be an instance of the HCP.
  We will construct an instance $x$ of the BERP (with weights in $\nat$) such that $x$ is a yes-instance if and only
  if $G$ contains a Hamiltonian cycle.
  To do so, we will construct a WFA with as many alphabet symbols as there are edges in $E$ and as many
  states as there are vertices in $V$.
  There will be a transition between two states if the corresponding vertices are joined by an edge in $G$
  and transitions have a unique alphabet symbol identifying them.
  More precisely,
  we construct the WFA \(\cA = (Q, \Sigma, \{M_a\}_{a\in\Sigma}, \initial, \final)\) where
  \(Q \coloneqq V\) and \(\Sigma \coloneqq \{\ell_{uv} \mid (u,v) \in E \}\).
  Let $\primes{n}$ be the set of the first $n$ prime numbers and
  let $\mu$ be an injective function $\mu : Q \rightarrow \primes{|Q|}$.
  For each \(\ell_{uv} \in \Sigma\), the transition matrix \(M_{\ell_{uv}}\) is defined as
  $\mu(v)$ at position $(u,v)$ and zero otherwise.
  Finally, we define \(\initial(u) \coloneqq 1\) and \(\final(u) \coloneqq 1 \) for each state \( u \in Q \).
  Note that transitions that point to the same state, share the same weight.

  Now, for $n \in \nat$, let $n\#$ denote the primorial of $n$ (the product of all prime numbers lower than or equal to $n$).
  Let $p_n$ be the $n$-th prime number.
  Let $x$ be the BERP instance defined as $x \coloneqq (1^{|Q|}, \cA, p_{|Q|}\#)$.
  Observe that, if $G$ has a Hamiltonian cycle, there is a word in $\cA$ of length $|Q|$ and weight exactly
  ${p_{|Q|}\#}$ (the word that results from concatenating the alphabet symbols associated to the edges of the Hamiltonian
  cycle) and so, $x$ is a yes-instance of the BERP language.
  On the other hand, if $x$ is a yes-instance, i.e., there exists a word in $\cA$ of weight ${p_{|Q|}\#}$ and length
  at most $|Q|$, we will argue that $G$ must contain a Hamiltonian cycle.
  This is because if such a word exists, it must be formed by exactly $|Q|$ symbols given the structure
  of $\cA$ and the fact that ${p_{|Q|}\#}$ admits a unique prime factorization.
  Moreover, given the injectiveness of $\mu$, those symbols correspond to transitions that point to all the states in $\cA$
  (and thus, to edges that point to all the nodes in $G$). Furthermore, the fact that the word's weight is non-zero
  implies that the edges corresponding to the transitions must form a path in $G$, but a path of length $|V|$ that visites
  all nodes exactly once must be a cycle.

  To conclude, we should make sure that $|x|$ is polynomial in the size of $|G|$ but that's the case because
  primes in $\primes{|Q|}$ can be represented with $2 \log{|Q|} = 2 \log{|V|}$ bits and so, $\cA$ can be represented by a polynomial number of
  bits. Furthermore, $\alpha \coloneqq {p_{|Q|}\#}$ can be represented with at most
  $2|Q|^2 = 2|V|^2$ bits. Therefore, we can conclude that $|x|$ is polynomial in the
  size of $|G|$.\footnote{These loose upper-bounds can be derived from the fact that
  $p_n < n (\log{n} + \log{\log{n}} )$ for $n > 5$ and that $\log{(n\#)} < n\left(1 + 1/(2\log{n})\right)$ for all $n \in \nat$~\cite{rosser1962}.}\\[-4ex]
\end{proof}

Also, observe that we could adapt the above reduction to produce a WFA with an alphabet of only two symbols, say $\Sigma = \{a,b\}$.
To do so, let $L \coloneqq \lceil\log_2{|E|}\rceil$ and assign to every edge $(u,v) \in E$ a unique identifier ${\sf id}_{uv} \in \{a,b\}^{L}$.
For each transition (alphabet symbol) in the automata from the above reduction, say $\ell_{uv}$, we can create
$L{-}1$ auxiliary states, and replace the transition by transitions between these auxiliary states forming a path, where each
transition is labeled according to the symbols in ${\sf id}_{uv}$ and all weights are defined as $1$, except for the last
one that keeps the weight of the original transition.
It is not hard to see that the reduction goes through, what implies that
the BERP is NP-complete even in the case of alphabets of two symbols.

Finally, we use a reduction from the BERP to prove that the BTRP is NP-complete as well.
We use a similar technique to the one used in the proof of undecidability of the Threshold Reachability problem
over probabilistic WFAs, given by Gimbert and Oualhadj~\cite{Gimbert2010}.
We adapt the technique in order to preserve integer weights.

\begin{theorem}
\label{thm:BTRP}
The BTRP with weights in $\integers$ is NP-complete.
\end{theorem}
\begin{proof}
  First, note that the BTRP is in NP since it admits a polynomial size certificate (that we can define in
  a similar way as in the previous proof).
  Now, let $(1^k, \cA, \alpha)$ be an instance of the BERP with weights in $\nat$.
  Let $\cB$ be the WFA defined as $(\cA \oplus 1 \ominus \alpha) \otimes (1 \oplus \alpha \ominus \cA)$ and observe that
  $\cB$ has integer weights.
  Furthermore, $\cB$ has a word of length bounded by $k$ and weight $\geq \!1$ if and only if $\cA$ has a word of length
  bounded by $k$ and weight exactly $\alpha$.
  This is because the function $f(t) \coloneqq (t {+} 1 {-} \alpha)(1 {+} \alpha {-} t)$ reaches it maximum value $1$ on $t = \alpha$.
  Finally, observe that the size of $\cB$ is polynomial in the size of $\cA$.
  Because the BERP with weights in $\nat$ is NP-complete (Theorem~\ref{thm:BERP}), the BTRP with weights in
  $\integers$ must be NP-complete.
\end{proof}

\section{Definition of \(\cA_E\)}
\label{app:construction}
In this section we give a definition of the WFA that recognizes the weighted regular language of function \(f\) (see Section~\ref{section:applications}).
Let \(\cA_E = (Q, \Sigma, \{M_a\}_{a\in \Sigma}, \initial, \final)\) with \(\Sigma \ud \{0, \ldots, 9, (, )\}, \initial \ud (1~0~0~0~0~0~1~1), \final \ud (-0.5~\mbox{-0.25}~~0~0.75~0~0~0.5~\mbox{-0.5})\) and 
\begin{displaymath}
M_(  \ud
\begin{bmatrix}
0 & 0 & 1 & 0 & 1 & 0 & 0 & 0\\[0.6pt]
0 & 0 & 1 & 0 & 0 & 0 & 0 & 0\\
0 & 0 & 0 & 0 & 0 & 0 & 0 & 0\\
0 & 0 & 0 & 0 & 1 & 0 & 0 & 0\\
0 & 0 & 0 & 0 & 0 & 1 & 0 & 0\\
0 & 0 & 0 & 0 & 0 & 0 & 0 & 0\\
0 & 0 & 0 & 0 & 0 & 0 & 0 & 0\\
0 & 0 & 0 & 0 & 0 & 0 & 0 & 0\\
\end{bmatrix},
\end{displaymath}

\begin{displaymath}
M_)  \ud
\begin{bmatrix}
0 & 0 & 0 & 0 & 0 & 0 & 0 & 0\\
0 & 0 & 0 & 0 & 0 & 0 & 0 & 0\\
0 & 1 & 0 & 0 & 0 & 0 & 0 & 0\\
0 & 0 & 0 & 0 & 0 & 0 & 0 & 0\\
0 & 0 & 0 & 1 & 0 & 0 & 0 & 0\\
0 & 0 & 0 & 0 & 1 & 0 & 0 & 0\\
0 & 0 & 0 & 0 & 0 & 0 & 0 & 0\\
0 & 0 & 0 & 0 & 0 & 0 & 0 & 0\\
\end{bmatrix}
\end{displaymath}
and
\begin{displaymath}
M_a  \ud
\begin{bmatrix}
0 & 1 & 0 & 1 & 0 & 0 & 0 & 0\\
0 & 1 & 0 & 0 & 0 & 0 & 0 & 0\\
0 & 0 & 1 & 0 & 0 & 0 & 0 & 0\\
0 & 0 & 0 & 1 & 0 & 0 & 0 & 0\\
0 & 0 & 0 & 0 & 1 & 0 & 0 & 0\\
0 & 0 & 0 & 0 & 0 & 1 & 0 & 0\\
0 & 0 & 0 & 0 & 0 & 0 & 0 & 0\\
0 & 0 & 0 & 0 & 0 & 0 & 0 & 1\\
\end{bmatrix}
\end{displaymath}
for every  \(a \in \{0, \ldots, 9\}\).
\bibliographystyle{ACM-Reference-Format}
\bibliography{main}

\end{document}